\documentclass{article}
\pdfoutput=1
\usepackage{epsfig}
\usepackage{url}
\usepackage{algorithm}
\usepackage{algpseudocode}
\usepackage{amsmath}
\usepackage{amssymb}
\usepackage{amsthm}
\usepackage{thmtools, thm-restate}

\usepackage{setspace}
\usepackage{graphicx}
\usepackage{subcaption}
\usepackage{amsfonts}
\usepackage[square]{natbib}

\newcommand{\eps}{\varepsilon}
\newtheorem{dfn}{Definition}
\newtheorem{thm}{Theorem}

\newtheorem{lem}{Lemma}

\title{Empirical Differential Privacy}
\author{Paul Burchard, Anthony Daoud, and Dominic Dotterrer}
\date{November 6, 2020}
\begin{document}

\maketitle

\begin{abstract}
We propose a privacy methodology for stochastic queries that requires no or reduced added noise in practice.
In this framework, privacy is estimated from the randomness in the data.
Unlike previous works on noiseless privacy, this empirical viewpoint avoids making any explicit assumptions about the random process generating the data.
Instead, we rely on only an observed set of databases in estimating privacy.
Assuming these databases are drawn i.i.d., we can estimate the privacy of any statistical query, even those with unbounded sensitivity.
Additionally, when the observed privacy is found to be insufficient for a query, we provide a methodology for drawing noise samples which guarantee desired privacy levels.
The primary applications of this work are centered around benchmarking and evaluating privacy protocols for which classical differential privacy either does not apply, is difficult to implement, or does not preserve required business logic.
\end{abstract}

\section{Introduction}

{\em Differential privacy\/}~\citep{dwork2006calibrating,TCS-042} is a popular, rigorous method for protecting individual information when providing statistics about a database.
Differential privacy offers a mathematical guarantee against membership disclosure in the sense that the privatized query responses are statistically indistinguishable whether any particular respondent's data is included or not.
This indistinguishability is specified by a {\it privacy budget}, $\eps$.
%What differential privacy guarantees is that the probability of any set of query results does not change multiplicatively by more than the privacy budget $\eps$ when individuals are added or removed.
In this classical setting, the database is deterministic and privacy is achieved by adding controlled random noise to the query response before reporting the results.
This mechanism has the property that an adversary can know everything about the database except the records of the individual in question, and that individual's information will still be protected.
From this starting point, various relaxations of the definition have been developed~\citep{desfontaines2020sok}.

In practice, the amount of noise that must be added to achieve this strong privacy guarantee is often large enough to limit the utility of the results~\citep{grining2016practical, piotrowska2015some}.
This is especially challenging for sum queries when the summands are unbounded, as such queries require infinite noise to be made private in the classical differential privacy setting.
This limitation is significant, particularly amidst common applications in financial services.
For this reason, ~\cite{Duan09,DuanACM09} and others~\citep{Bhaskar12,Grining17} have introduced the idea of leveraging the randomness in the data itself to help ensure privacy, resulting in so-called {\em noiseless privacy\/} mechanisms.
These methods are applicable when the adversary has imperfect knowledge of the database;
an individual's contribution to a query result is hidden by the unknown random values drawn for the other individuals in the database.

These noiseless methodologies have several limitations.
First, they involve explicit and restrictive assumptions about the probability distribution of the data;
in particular, database records must be drawn independent and identically distributed from a certain known distribution.
Second, they provide no savings on the added noise required when the observed privacy does not meet levels required of the application.

In what follows, we present a notion of {\em empirical differential privacy\/}, which estimates the inherent privacy of a query applied to a random database, but does not place explicit assumptions on the distribution of the row data.
%In many cases where we can think of the database as random, however, we do not only have a single database, but rather multiple samples of the database.
%In such cases, we can replace explicit statistical assumptions with the empirical probability distribution of the data, resulting in a notion of {\em empirical differential privacy\/} for deterministic or randomized queries.
We also explore how to modify this definition in the face of limited adversarial knowledge of the data.

\subsection{Contributions}

\begin{itemize}
   \item A privacy framework known as {\em empirical differential privacy\/} that makes use of empirical randomness in the data to help provide privacy.
    	 We do not require any assumptions over the database distribution, nor do we require that records in the database are independent.
   \item A new metric called {\em total risk\/} ($\delta^*$) which provides an intuitive measure of the risk that a privacy breach to {\em some\/} individual will occur.
   \item A methodology for providing additional noise to the query to meet required privacy standards.
         By taking advantage of the randomness in the data, the approach uses more concentrated noise samples than those that would be provided by classical differential privacy.
   \item Discussion of how post-processing, composition, and adversarial knowledge can be handled in this framework.
   \item Discussion of numerical considerations inherent in the method.
\end{itemize}

\subsection{Related Work}

As previously mentioned, ~\cite{DuanACM09} was one of the earliest to explore a noiseless approach that satisfies a relaxed version of differential privacy making use of the randomness in the data.
They showed that the privacy of sum queries could be ensured under certain restrictive statistical assumptions on how the random database was generated (with independent, identically distributed rows).
A more broadly applicable noiseless model was established by ~\cite{Bhaskar12}.
They are the first to define the notion of ``noiseless privacy,'' and they handle more classes of functions as well as introduce an adversarial model.
To achieve their results, they make assumptions about the distribution from which database entries are drawn and require that most rows are independent of each other.

This approach was then expanded upon by ~\cite{Grining17}.
They allow for dependent data and provide privacy guarantees in non-asymptotic situations, a vital result for more practical use-cases in data privacy.
To calculate privacy when the data is dependent, they require an estimate of the fraction of users compromised by an adversary, as well as the largest dependency "neighborhood."
Such assumptions are valid in certain scenarios, for example when releasing a query over many databases where one is known to be compromised.
However, it may often be difficult to estimate these quantities, which is a limitation that our privacy framework avoids.

A distinct but related idea is {\em random differential privacy (RDP)}, where the randomized query must be differentially private in the usual sense only over a likely set of databases, not the entire set of databases~\citep{Hall_Wasserman_Rinaldo_2013}.
In this setting, the privacy guarantee protects only for these probable query responses.
Furthermore, the randomness in the data is not used to estimate privacy at all.

~\cite{rubinstein2017pain} then build on {\em RDP} in a manner that is more connected with our approach.
Similar to our situation, they assume the sensitivity of the query may be difficult to bound analytically. 
However, they further assume that one is able to sample databases at will from a distribution whose KL divergence from the true distribution of databases is bounded.
They then estimate the sensitivity of the query by evaluating the query on each of the sample databases with and without a single row.
While this allows one to achieve any target {\em RDP} confidence level, it is often practically too difficult to determine an underlying distribution from which databases should be sampled.
Our approach better handles this issue as we only rely on the true set of databases and require no distributional assumptions.
From this, we determine the privacy level inherent in the databases.

We are aware of a previous unrelated use of the term ``empirical differential privacy'' initiated by ~\cite{Abowd13}.
In this paper, there is a single dataset which the authors model as sampled from a statistical process with unknown parameters.
The parameters are estimated using a bayesian paradigm.
They then estimate the privacy of the parameters by dropping an individual out of the dataset before re-calculating the posterior over the parameters.
By comparing the posteriors with and without an individual, they are able to estimate the overall privacy of the model parameters.
While similar in philosophy to our approach, they require an assumption of the statistical model that generates the data as well as the prior over model parameters.
Our approach requires no such assumptions.

\section{Problem Statement}
In our setting, a database is a finite non-empty set of records.
Each record consists of a universal identifier which identifies the respondent or client to which the record refers, along with any amount of numerical data associated to that record, such as timestamps, transaction amounts, additional IDs that could identify location, transaction type and so on.
Each database may have an arbitrary number of records associated to a single client.
The reader may keep in mind the example of a database of all customer transactions on a given calendar day.

We assume that we have independently drawn a collection of databases, $X^1, \dots, X^n$, from an unknown but fixed probability distribution of databases, $X \sim \chi$, amongst a set of all databases, $D$.
For example, the sample databases might arise from random subsampling of a single, large master table.
Another important example arises in the idealized setting in which customer transactions are time-independent and each sample database consists of all the transactions of a particular day, with each day providing a different random sample database.

We further consider a statistical query $f: D \rightarrow R$ which is to be studied.
The query may be a deterministic function of the database, or it may introduce independent noise as is often the case with traditionally differentially private statistical queries.
The result, $f(X)$, is a real valued random variable for which we have some samples $f(X^1), \dots, f(X^n)$.

Our objective is ensure a quantitative level of privacy for the respondents in the sense that the log-likelihood of $f(X)$ given that Client $k$ contributed to database X mutually bounds the log-likelihood of $f(X)$ given that Client $k$ did not contribute. 
Letting $X_k$ represent the distribution of databases without client $k$, we can write this statement as
\[
	\vert \log( \mathbb{P} ( f(X) \vert X ) )  - \log( \mathbb{P} ( f(X) \vert X_k ) ) \vert < \epsilon
\]

It is entirely possible that, by the very nature of $f$, $X \sim \chi$, and $D$, that this condition already holds.
In this case, our objective is to empirically verify the condition.
In the case that the condition does not naturally hold, our objective is to quantify the failure probability $\delta$ and to modify $f$ in much the same way that is prescribed by conventional differential privacy to reduce $\delta$.

%While any statistical query may be chosen, the use-cases most suited to our framework are those in which the sensitivity of $f$ is either unbounded or difficult to bound analytically.
%Let $X \in D$ be a database drawn from an unknown distribution of databases $\chi$.
%We assume that a statistician posing the query provides a set of i.i.d. databases $X^1, \ldots, X^n \sim \chi$.

%The statistician would like to release the result of the query $f(X)$.
%However, they also would like an estimate of the privacy level $\epsilon$ of the release as well as the failure probability $\delta$.

\section{Empirical Approach}

In the present work, we take an empirical approach to estimating individual privacy.
At a high level, our approach to determining privacy works by first running the query on each sample database to obtain an original set of query results.
We then drop an individual $i$ from all databases before re-running the query to obtain a second set of results without $i$.
From each of these two sets, we infer a corresponding probability density over the space of query responses.
By comparing these two density functions, we can estimate the individual risk ($\delta_i$) by integrating over the region where the ratio of
the densities differ by greater than $\exp(\eps)$.
Finally, we repeat this process for all individuals in the dataset, and return the max $\delta_i$ as the overall failure probability $\delta$.

Before defining the privacy guarantees of such an approach, we first formulate a suitable definition of differential privacy in the stochastic database paradigm.

%\begin{rmk}
%In this description, we have considered two probability measures--the distribution of query results with individual $i$ included and the distribution with the same individual excluded.
%However, the probabilities in question can be any conditional distributions in which event $E$ occurs or not.
%By analogy, we could, for example, consider the event $E$ which is defined by an individual contributing to the database.
%In this case, the probability measures in question are the conditionals $P(q(X)|E)$ and $P(q(X)|\neg E)$.
%Alternatively, we could consider other events such as ``the data was collected on Christmas day'' or
%\end{rmk}

\begin{dfn}
\label{dfn.noiseless_dp}
Let $X \sim \chi$ be a random variable taking values in the set of databases, $D$. 
Let $f: D \rightarrow R$ be a deterministic or randomized query.
Assume that individuals $i$ have a unique identity across $D$.
Then let $X_i \sim \chi_i$, $X_i \in D$ denote the resulting distribution of databases with the data of individual $i$ removed.
Denote by $P$ the probability measure of the query results $f(X)$, including both $f$ and $X$ as sources of randomness,
and by $P_i$ the resulting probability measure of the query results $f(X_i)$, including both $f$ and $X_i$ as sources of randomness.
For $\eps > 0, \delta \in [0,1]$, we say that $f$ satisfies {\em inherent $(\eps,\delta)$-differential privacy\/} relative
to $\chi$ iff, for all individuals $i$, all measurable sets of potential query results $E$
\begin{eqnarray*}
P_i(E) &\le& \exp(\eps)P(E) + \delta \\
P(E) &\le& \exp(\eps)P_i(E) + \delta .
\end{eqnarray*}

\end{dfn}

\begin{lem}
Suppose $f$ is a classically $(\eps,\delta)$-differentially private randomized query.
Then for any random distribution $\chi$ of the database $X$, $f$ is inherently $(\eps,\delta)$-differentially private.
However, the converse is not true.
\end{lem}

This is because classical differential privacy implies $P_i(E|X) = P(E|X_i) \le \exp(\eps)P(E|X) + \delta$ and $P(E|X) \le \exp(\eps)P_i(E|X) + \delta$ for all $X$.
We can then average over the distribution $\chi$ of $X$.
Note that classical differential privacy protects the individual, which in this instance may be multiple databases rows.

In practice, the probability distribution $\chi$ is rarely known.
Therefore, the probability measures $P$ and $P_i$ must be estimated from independent samples of $X$.
We denote these estimators as $\widehat{P}$ and $\widehat{P_i}$ respectively.
We discuss the determination of these estimators below.

\begin{dfn}
\label{dfn.privacy}
Assume we have sample databases $X^1, \dots, X^n$ independently drawn from $\chi$.
We say that a deterministic or randomized query $f$ satisfies {\em empirical differential privacy\/} if the estimating distribution $\widehat{P}$ of $f(X)$,
and corresponding marginal estimator $\widehat{P_i}$ of $f(X_i)$,
satisfy inherent differential privacy, as above.
(The samples used to construct the empirical measure $\widehat{P}$ include the sample databases $X^1,\ldots,X^n$,
and if $f$ is a randomized query, multiple samples of $f$ drawn for each sample database.)
\end{dfn}

To practically implement this definition, we reformulate the privacy condition in terms of the estimator probability density functions, which we denote $\hat{p}$ and $\hat{p_i}$, respectively.
Moreover, we need a means of estimating the probability of failure $\delta$ for a specified $\epsilon$.
Our approach is to numerically integrate the region of the query response space where the privacy guarantee is not satisfied.
We do this for each individual $i$ to give an individualized probability of failure $\delta_i$.
We then set $\delta$ to be the largest of these across all individuals, which is stated more formally as follows:

\begin{thm}
\label{thm.densities}
Let $\hat{p}$ and $\hat{p}_i$ be probability density estimators with and without individual $i$,
inferred from the sets of samples $\{f(X^1), \ldots f(X^n)\}$ and $\{f(X_i^1), \ldots, f(X_i^n)\}$, respectively.
Then the statistical query $f$ is empirically $(\eps,\delta)$-differentially private if, for all $i$, the densities $\hat{p}$ and $\hat{p}_i$ differ by a factor of at most $\exp(\eps)$, with the exception of a set $\overline{A}$ on which the densities exceed that bound by a total of no more than $\delta$,
where
\begin{eqnarray*}
	\delta_i &=& \max \left( \int_{\overline{A}} (\hat{p}(x) - \exp(\eps) \hat{p}_i(x))_{+} \,dx , \int_{\overline{A}} (\hat{p}_i(x) - \exp(\eps) \hat{p}(x))_{+} \,dx \right), \\
	\delta &=& \max_{\forall i} \delta_i.
\end{eqnarray*}

\end{thm}

To prove this claim,
let $A$ be the set where the densities differ by less than the required factor (i.e., $\hat{p} \le \exp(\eps)\hat{p}_i$ and $\hat{p}_i \le \exp(\eps)\hat{p}$), and $\overline{A}$ be the remaining set.
Then for any set of potential query results $E$,
\begin{eqnarray*}
\lefteqn{\hat{P}_i(E) - \exp(\eps)\hat{P}(E)} \qquad && \\
&=& \int_E \hat{p}_i(x) \,dx - \int_E \exp(\eps) \hat{p}(x) \,dx \\
&\le& \int_{E \cap A} \exp(\eps) \hat{p}(x)\,dx + \int_{E \cap \overline{A}} \hat{p}_i(x)\,dx \\
&&\qquad\mbox{} - \int_{E \cap A} \exp(\eps) \hat{p}(x) \,dx - \int_{E \cap \overline{A}} \exp(\eps) \hat{p}(x)\,dx \\
&\le& \int_{E \cap \overline{A}} (\hat{p}_i(x) - \exp(\eps) \hat{p}(x))_{+} \,dx \\
&\le& \max\left( \int_{\overline{A}} (\hat{p}(x) - \exp(\eps) \hat{p}_i(x))_{+} \,dx , \int_{\overline{A}} (\hat{p}_i(x) - \exp(\eps) \hat{p}(x))_{+} \,dx \right) \\
&=& \delta_i \leq \delta .
\end{eqnarray*}
The same works with the roles of $\hat{P}$ and $\hat{P}_i$ reversed, proving that the empirical probability measures obey $(\eps,\delta)$-privacy when $\delta$ is chosen to satisfy the above inequalities.

Next we investigate the impact of post-processing and composition on empirical differential privacy.

\begin{lem}
\label{lem.post_process}
Let $f$ be an inherently $(\eps,\delta)$-differentially private query on a random database $X \sim \chi$.
Consider an arbitrary deterministic or randomized function $g$.
Then $g \circ f$ is inherently $(\eps,\delta)$-differentially private.
\end{lem}

To prove this, let $g$ be deterministic and consider an event $E$. Then

\begin{eqnarray*}
P[ g(f(X)) \in E] &=& P[ f(X) \in g^{-1}(E)] \\
					&\leq& e^{\eps} P[ f(X_i) \in g^{-1}(E)] + \delta \\
					&=&  e^{\eps} P[g(f(X_i)) \in E] + \delta \\
\end{eqnarray*}

where $g^{-1}(E) = \{ x \in R | g(x) \in E \}$ denotes the pre-image of $E$ under $g$.
If $g$ is randomized, this inequality can be averaged over the different values of $g$.

For empirical privacy, due to the empirical approximation $\widehat{P}$ we are making to the true distribution $P$ of query results,
this post-processing result only holds approximately.
Specifically, $g \circ f$ will be empirically $(\eps',\delta')$-differentially private, where $(\eps',\delta')$ exceeds $(\eps,\delta)$ by any fixed margin with probability approaching zero as the number of samples $g(f(X))$ grows.

Composition of queries $f_1$ and $f_2$ works differently in the empirical than the classical setting.
If we wish to perform another query $f_2$ on the same set of databases $X^1, \dots, X^n$, we must jointly put both queries through the above privacy validation process.
Thus, we look at the two-dimensional empirical probability densities of the joint query $f = (f_1, f_2)$.
Without added noise, it is quite possible for $f_1$ and $f_2$ to be empirically private, but $f = (f_1,f_2)$ to precisely reveal individual information.

As an example, suppose $f_1$ is a query requesting the sum over a column of a database, and $f_2$ is a query requesting the sum over the same column except without individual $i$. 
If we consider these queries independently, they can be arbitrarily private even though together they clearly reveal $i$'s contribution to the sum.

In this framework, we are also able to better quantify the risk of privacy loss.
This is because the $\delta$ in the privacy bounds, being only the worst case for a single individual, often understates the loss of privacy.
Instead, we can determine the probability of failure for each individual, and then combine them to calculate the overall risk of a privacy breach.
We define this more useful metric as total risk ($\delta^*$):

\begin{dfn}
Given an $\eps$, for each individual $i$, let $\delta_i$ be the individual failure probability as defined in Theorem~\ref{thm.densities}.
Then the {\em total risk\/} $\delta^* = 1 - \prod_i ( 1 - \delta_i )$ is an estimate of the probability that some individual will have their privacy compromised by more than $\eps$.
\end{dfn}

This heuristic metric is accurate when the probability of data leakage for individuals is independent, but will underestimate privacy risks when individuals are known to be dependent.
However, in a scenario in which the relationship between individuals is uncertain, this metric can still be useful as it reveals the privacy
risk if the individuals were to be truly independent.
This allows for a more intuitive understanding of the minimum risk of a privacy breach in that scenario.

Although Theorem~\ref{thm.densities} shows how to obtain privacy guarantees for deterministic queries, it may sometimes still be useful to introduce randomized queries, as in the classical setting.
In particular, if we find that the desired deterministic query is not empirically private enough, we can add a smaller amount of noise to the query in the same manner as classical differential privacy, such that combining samples of this noise with the natural randomness of the sample sequence of databases achieves the desired empirical differential privacy.
We discuss this further in Section~\ref{sec.noise}.

\section{Adversarial Setting}

So far, we have assumed that the adversary only has access to the total number of databases and the published values of the statistical query.
Classical differential privacy assumes that the adversary could have access to all of the underlying data except for the one individual whose privacy we are trying to protect.
More typically, there is an intermediate situation in which the adversary knows some of the data.
Moreover, there may be statistical dependency between the data known to the adversary and the rest of the unknown data~\citep{Grining17}.

In the empirical approach, we can handle all of these scenarios in a uniform way by looking at empirical conditional probability distributions instead of absolute distributions.

\begin{dfn}
Let $Z$ denote the set of data the adversary knows.
Given a deterministic or randomized query $f(X)$ on a random database $X$, and another query $g(X)$ that reveals information known to an adversary,
let $P(y|Z)$ denote the conditional probability measure determined by $y = f(X)$, conditional on knowledge $g(X)\in Z$ of the adversary information.
Similarly, let $P_i(y|Z)$ denote the conditional probability measure determined by the sample $f(X_i)$, conditional on knowledge $g(X)\in Z$ of the adversary information.
Then we say that $f$ is {\em inherently $(\eps,\delta)$-differentially private with adversary $g$\/} iff for all individuals $i$, all sets of potential query results $E$, and all sets of potential adversary information $Z$,
\begin{eqnarray*}
P_i(E|Z) &\le& \exp(\eps)P(E|Z) + \delta \\
P(E|Z) &\le& \exp(\eps)P_i(E|Z) + \delta .
\end{eqnarray*}
\end{dfn}

The empirical version of this definition simply substitutes in the estimated empirical measure $\widehat{P}(y|Z)$.

This calculation generally requires a larger sample of databases $X^1, \ldots, X^n$ because we are splitting the total data across values of the adversary information.
To help with this, we can practically restrict the $Z$ we test to larger buckets of the adversary data, at some cost in the accuracy of the privacy metrics.
Nevertheless, this approach is only practical if the amount of adversarial information is limited.

For example, one typical way the adversarial setting comes into play is that even though the data is not public, each individual contributing to the database knows their own contribution.
Thus we need to test differential privacy on the distributions conditional on each individual's contribution.
Note that this will not usually satisfy the full definition above---it will fail for the individual $i$ in question, which is to be expected given perfect knowledge of that individual's contribution.

\section{Numerical Considerations}
\label{sec.numerical}

Although we do not make any explicit statistical assumptions in this empirical approach, there are two important implicit assumptions on which the concluded privacy depends:
\begin{description}
\item[Independent Databases]
To ensure that the empirical probability densities accurately represent the true underlying distribution, the sequence of databases should be a sample of independent draws from the hypothetical distribution of databases.
No other specific statistical assumptions about the distribution of the databases, or statistical properties of the rows within each database, need be made.

Of course, it is essential that the condition of independence of database samples be verified beforehand through the acceptance of an appropriate null hypothesis.

\item[Representative Individuals]
Since the empirical privacy criterion takes a maximum over all individuals in the sample population, it is important that the population contain representative individuals from across the distribution of individuals.

For example, the work of ~\cite{Duan09} makes the explicit assumption that all individuals are statistically identical, and proves noiseless privacy of sum queries under that assumption.
The present method does not require that each individual's data is independent and identically distributed.
However, for statistical inference to be valid, observed data should reflect the population of individuals broadly, which is to say that the samples should provide sufficient coverage of the modes of individual contribution.
\end{description}

For clarity's sake, we now show an algorithm for calculating empirical privacy (Algorithm~\ref{alg:EmpiricalPrivacy}).
This algorithm returns the probability of failure $\delta$ and total risk $\delta^*$ for a specified $\epsilon$.
Critical to this approach is the choice of method for estimating the density functions $\hat{p}, \hat{p}_i$.
Along with this, one should be mindful of the integration method, be it numerical or otherwise, that is chosen in Algorithm~\ref{alg:InferPrivacyRisk}.
We discuss considerations for these concerns in the following paragraphs.

\begin{algorithm}[H]
	\setstretch{1.35}
	\caption{EmpiricalPrivacy}
	\hspace*{\algorithmicindent}\textbf{Input:} databases $X^1, \ldots, X^n$, statistic function $f$, and privacy level $\epsilon > 0$
	\label{alg:EmpiricalPrivacy}
	\begin{algorithmic}[0]
		\State {$\delta, \overline{\delta^*} \gets 0, 1$}
		\ForAll{individuals $i$}
			\For{$j \gets 1$ to $n$}
				\State {$f^j \gets f(X^j)$}
				\State {$X_i^j \gets X_j$ $\setminus$ rows where $i$ is present}
				\State {$f_i^j \gets f(X_i^j)$}
			\EndFor
			\State {$\delta_i \gets InferPrivacyRisk\left(\left\{f^1, \ldots, f^n\right\},
					\left\{f_i^1, \ldots, f_i^n\right\}, \epsilon \right)$}
			\State {$\delta \gets \max \left( \delta, \delta_i \right)$}
			\State {$\overline{\delta^*} \gets \overline{\delta^*} \cdot \left( 1 - \delta_i \right)$}
		\EndFor
		\State \Return $\delta, 1-\overline{\delta^*}$
	\end{algorithmic}
\end{algorithm}

\begin{algorithm}[H]
	\setstretch{1.35}
	\caption{InferPrivacyRisk}
	\label{alg:InferPrivacyRisk}
	\hspace*{\algorithmicindent}\textbf{Input:} sets of points  $q$, $q_i$ that are query results with and
					without individual $i$, and privacy level $\epsilon$
	\begin{algorithmic}[0]
		\State {$\hat{p}(x) := DensityEstimator(x | q)$}
		\State {$\hat{p}_i(x) := DensityEstimator(x | q_i)$}
		\State {$g(x) := \hat{p}(x)- e^\epsilon*\hat{p}_i(x)$}
		\State {$h(x) := \hat{p}
		_i(x)- e^\epsilon*\hat{p}(x)$}
		\State {$g_{pos} \gets$ integral of $\max \left( g(x), 0 \right)$}
		\State {$h_{pos} \gets$ integral of $\max \left( h(x), 0 \right)$}
		\State \Return {$\max \left( g_{pos}, h_{pos} \right)$}
	\end{algorithmic}
\end{algorithm}

There are a variety of methods available to estimate the probability of a query result, $\hat{P}$, from the sample databases. 
The easiest is to generate the empirical distribution function of the samples, which can be done simply by sorting the sample answers.
However, our use case requires a $DensityEstimator$ function in order to determine $\delta$.
To obtain that, one can numerically differentiate the empirical cumulative function.
To balance the bias and variance of the estimate, the numerical differencing should be taken across approximately the square root of the number of sample points.

Another standard method is {\em kernel density estimation\/} (KDE) ~\citep{Silverman86}, which has the benefit of use-case adaptability; the design of the kernel can take into account domain knowledge about the data and application.
Note that the classic approach is to use a kernel with a fixed bandwidth.
However, determining empirical differential privacy is highly dependent on the estimates of the distribution in sparse regions.
A slight overestimation of tail densities can lead to a large overestimation of privacy.
For this reason, our use case requires variable bandwidths, as described in~\cite{breiman1977variable}. 
This is because the quality of the privacy estimate hinges critically on the density estimate in low data regions, and variable width kernels are better able to capture the density of sparse data.

In practice, our approach has been to articulate a model space parametrized by the choice of kernel and the variable bandwidth of that kernel.
We then choose the model that maximizes the out-of-sample log likelihood.

\section{Noise Addition}
\label{sec.noise}

Up to this point, we've only considered using empirical differential privacy to estimate the risk of releasing the responses to a query.
However, we can add noise in a manner analogous to classical differential privacy to achieve a specified privacy level $\epsilon$.
Suppose a data curator calculates the empirical differential privacy of a query and determines that it is insufficient.
The curator instead would like to achieve a specified privacy level $\epsilon$ by adding noise to the query distribution.
When a new query response is drawn from the underlying distribution--whose density function we denote $p(x)$--the curator then draws an independent sample from the noise distribution--whose density function we denote $h(x)$--and adds it to the query response.
This additionally randomized query response is distributed according to the density function $p \star h (x) = \int h(y - x)p(y) dy$, i.e. the convolution of $p$ and $h$.

From the curator's perspective, they only need specify $\epsilon$ and a reasonable methodology to estimate the true distribution of query responses, $p(x) dx$.
We will denote the curator's best estimator as $\hat{p}$.

Our end goal is to find the noise kernel $h$ such that $\hat{p} \star h$ is empirically differentially private with respect to $\hat{p}_i \star h$ for each sensitive entity $i$.
When the curator receives a new draw from $p$, an independent sample from $h$ is added and the result is released.
Thereby, a noise kernel $h$ is more desirable if it has smaller polar moment, $\int |x| h(x) dx$, which is the same as the expected error of the released sample.

First, we will propose candidate distributions for $r=\hat{p} \star h$ and $r_i = \hat{p}_i \star h$.
We will employ the Laplace kernel density estimators with scale $\lambda$.
We select $\lambda$  to guarantee $\epsilon$-empirical differential privacy.
	\begin{restatable}{proposition}{bandwidth}
	\label{prop.bandwidth}
		Let $A$ and $B$ denote two finite sets of real numbers. Let $d_H (A,B)$ denote the Hausdorff distance between $A$ and $B$,
		\begin{gather*}
            d_H (A,B)  = \max \left( \max_{\forall a \in A} \min_{b \in B} |a - b|, \max_{\forall b \in B} \min_{a \in A} |a - b|,\right)
        \end{gather*}
		Further, let $\alpha$ be the Laplace kernel density estimator of the set $A$ (resp. $\beta$ for $B$) with scale $\lambda$.
		Then $\lambda \geq \frac{d_H(A,B)}{\epsilon}$ implies that
	\[
		\vert \log \alpha(x) - \log \beta(x) \vert \leq \epsilon \text{    for all    } x
	\]
    \end{restatable}

See the appendix for the proof of this proposition.

Let $\hat{P}$ and $\hat{P}_i$ respectively denote the empirical sample of query responses with and without the sensitive subject's data included.
We compute $\lambda_i = \frac{d_H(\hat{P}, \hat{P}_i)}{\epsilon}$ and let $\lambda = \max_i \lambda_i$.
With this $\lambda$ we have now determined the estimators $r$ and $r_i$ described above.
What remains is to solve for the noise kernel by {\it deconvolution} \citep{diggle1993fourier}.
Specifically, we apply the Fourier transform to both sides of $r = \hat{p} \star h$ to obtain:

\begin{equation}
    h(x) = \mathcal{F}^{-1} \left( \frac{\mathcal{F} r}{ \mathcal{F}\hat{p} } \right)
\end{equation}

where $\mathcal{F}$ is the Fourier transform.
We note for the reader that the curator's inference step has heretofore been unconstrained, but at this stage we will require that $\hat{p}$ have a Fourier transform that does not vanish on any frequency.
But we also arrive at a more substantial issue: How can we be sure that the $h$ we obtain from $r = \hat{p} \star h$ and from $r_i = \hat{p}_i \star h$ are in fact the same?
Here we constrain the generation of $\hat{p}$ to being a kernel density estimator with user-specified kernel, $k$, so that
\[
\begin{array}{lll}
	\hat{p} (x) &=& \frac{1}{n} \sum_{y \in \hat{P}} k(x - y) \text{ and} \\
	r (x) &=& \hat{p} \star h \\
	      &=& \frac{1}{n} \sum_{y \in \hat{P}} (k \star h)(x - y)
\end{array}
\]
Since $r$ is itself a Laplace kernel density estimator, the problem of computing $h$ therefore amounts to the deconvolution:
\begin{equation}
    h(x) = \mathcal{F}^{-1} \left( \frac{\mathcal{F} \ell_\lambda}{ \mathcal{F} k} \right)
\end{equation}

where $\ell_\lambda (x - y) = \frac{\exp \left( \frac{-\vert x - y \vert}{\lambda} \right)}{2 \lambda}$ is the Laplace kernel of scale $\lambda$ determined above.

We then arrive at a partial noising mechanism which achieves the required privacy level.
	\begin{thm}
	\label{thm.noisy_edp}
	Assume we have sample databases $X^1, \dots, X^n$ independently drawn from $\chi$.
	Furthermore, assume a data curator provides a deterministic or randomized query $f(X)$, a desired privacy level $\epsilon$,
	and a distributional inference kernel $k$, so that $\hat{p}$ is written $\hat{p}(x) = (k \star \hat{P})(x) = \frac{1}{n} \sum_{j \leq n} k(x - f(X^j))$.
	Let $h(x) = \mathcal{F}^{-1} \left( \frac{\mathcal{F} \ell_\lambda}{ \mathcal{F} k} \right)$ as above.
	Then the mechanism defined by
		\[
			z(X) = f(X) + y \
		\]
	where $y \sim h$, preserves $\epsilon$-empirical differential privacy.
	\end{thm}

\begin{figure}
	\begin{center}
	\setlength{\epsfxsize}{3in}
	\epsffile{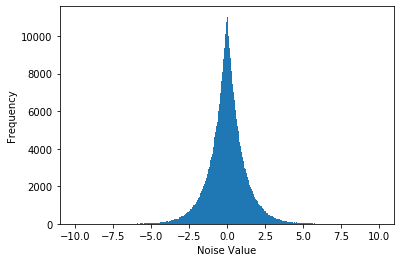}
	\end{center}
	\caption{The noise distribution obtained from a curator specified Gaussian KDE.
				The curator applied Silverman's rule of thumb to obtain a kernel width of $1.06 \cdot \hat{\sigma} \cdot n^{-\frac{1}{5}} \sim 0.002$.
				Analysis of the Hausdorff distance of samples yielded $\lambda \sim 0.89$.
				The figure shows a histogram of 1,000,000 samples.}
	\label{fig.noise_kernel}
\end{figure}

At this point, we should speak to the importance of choosing $k$.
Obviously, $k$ can always be chosen to sufficiently blur the results and justify adding no additional noise.
But the ointent of $\hat{p}$ is to be a good estimator of the probability of future samples, which is why we suggest that $k$ be chosen as a likelihood maximizer of a holdout sample in order to ensure a ``fair'' $k$ has been chosen.

We will conclude the section with some analytical considerations arising from the operation of convolution by $k$ (denote the operator by $Kf = k \star f$).
The operator $K$ has an invariant action on the subspace of probability densities (a high dimensional simplex), but the restricted action of $K$ on this space is not surjective.
For our purposes, however, all that is required is that $\ell_\lambda$ is in the image of $K$.
For standard kernels, such as Gaussian and Laplace, this can be achieved provided that the scale of of $k$ is small relative to $\lambda$.

As an example noise distribution, see Figure~\ref{fig.noise_kernel}.
The diagram shows a histogram of samples from the resulting noise kernel.

\section{Applications}

Empirical differential privacy is applicable for datasets in which component databases have some shared notion of identity and are drawn from the same distribution.
A dataset of this nature could, for example, come from a utility company that has collected household electricity usage data over multiple years.
Each year's data could be considered a separate database on which a statistic of interest is computed.
In this case, the entity to be protected is an individual's contribution to the statistic.

In another scenario, a mall could collect wifi access point data about which phones connected to the system.
Each access point's data would compose a database, and statistics on each access point could be released while protecting the privacy of the users.

A couple other examples of interest include  queries that select a random subset of rows from some master dataset, and queries that inject noise at the local level to the query response.
The former occurs ubiquitously in machine learning, where the query is to build a model given a random train-test split of a dataset.
In this instance, each random subset of rows can be thought of as a database, and the privacy inherent in the model may be estimated.
In the latter case, an upstream non-rigorous privacy protocol may inject noise into the data before computing the query.
If we have no knowledge of the algorithm that injects that noise (we can't see upstream), we may want to evaluate empirically how much that protocol privatizes the query.

While it is easy to imagine further use cases, public datasets rarely have the requisite identifying information as it is an obvious privacy violation.
Even datasets with pseudo-anonymous identifiers are difficult to find because these can be deidentified (for a notorious example, see~\cite{taxis}).
As a result, our dataset will serve as a proxy for more realistic use cases and is intended to be an illustrative example of how to achieve empirical differential privacy.

We are using the Global Precipitation Climatology Project's monthly precipitation dataset from 1979-2019~\citep{gpcp_data}.
In this example, each data point is associated with a location on a latitude and longitude grid, as well as a value for the average daily precipitation (mm/day) at that location for that month.
We consider each year's worth of data as a database, and we would like to release the annual precipitation averaged across all locations.
We will be attempting to protect each location's contribution to this average.
This means that each location is an individual in our framework.
To connect back to the household electricity usage example, we can think of each location in this dataset as analogous to an address of a home.
We can also view precipitation values as analogous to each home's electricity usage.

We are given the average daily precipitation in each month per location, which we convert to the the annual total precipitation per location.
Our dataset then consists of 41 databases, each of which contain the same 10368 individuals.
The average precipitation in each year is shown in Figure~\ref{fig.year_vs_precip}.
To get a sense of the spread of precipitation values by location, the standard deviation of precipitation values within each year is on average 686 mm.

Normally, one should verify statistically that the each database is i.i.d., a necessary condition for empirical differential privacy.
This can be done by accepting a null hypothesis based on a standard statistical test (e.g. T-test).
That test is omitted here as this example is only a proxy for a real application.

\begin{figure}
	\begin{center}
	\setlength{\epsfxsize}{3in}
	\epsffile{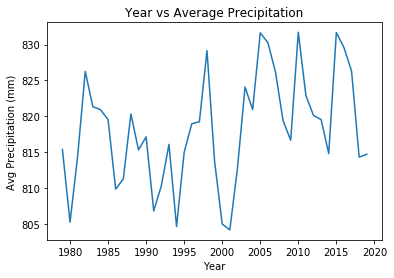}
	\end{center}
	\caption{Annual precipitation values averaged across all locations in each year.}
	\label{fig.year_vs_precip}
\end{figure}

Given that we know little about the shape of the probability distribution, we will use Kernel Density estimation with Laplace Kernels.
In general, we recommend maximizing the out-of-sample log likelihood to determine the appropriate bandwidth.
Note that this step does not consume any privacy budget, as the empirical differential privacy paradigm assumes the consumer of the calculated $\delta$ is trusted.

With this setup, and with varying $\eps$, the performance characteristics for $\delta$ and $\delta^*$ are shown in Figures~\ref{fig.eps_vs_delta} and~\ref{fig.eps_vs_total_delta}.
\begin{figure}
	\centering
	\begin{minipage}[c]{.45\textwidth}
	\centering
	\includegraphics[width=\linewidth]{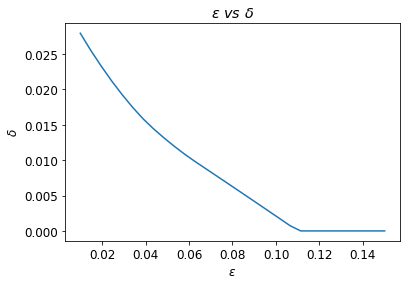}
	\caption{$\delta$ values for varied $\eps$}
	\label{fig.eps_vs_delta}
	\end{minipage}
	\hfill
	\begin{minipage}[c]{.45\textwidth}
	\includegraphics[width=\linewidth]{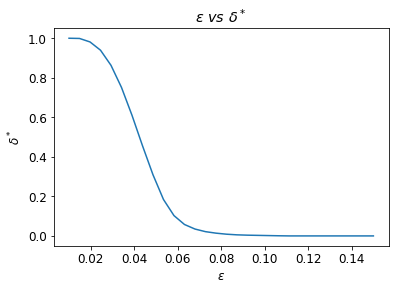}
	\caption{$\delta^*$ values for varied $\eps$}
	\label{fig.eps_vs_total_delta}
	\end{minipage}
\end{figure}

\begin{figure}
	\begin{center}
	\setlength{\epsfxsize}{3in}
	\epsffile{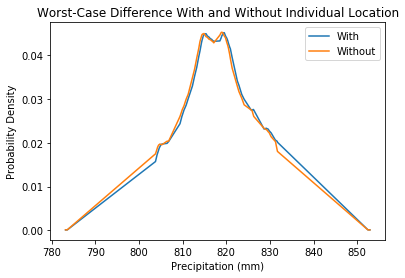}
	\end{center}
	\caption{Empirical probability density of the average annual precipitation dropping the individual with the biggest $\delta$, compared with the empirical probability density including all individuals.}
	\label{fig.worst_density}
\end{figure}

Based on Figure~\ref{fig.eps_vs_delta}, it is clear that all individuals are protected when $\epsilon$ is $0.12$ or greater.
This is true even though the query is unbounded.
Moreover, Figure~\ref{fig.eps_vs_total_delta} shows a clear measure of how likely it is that any individual will have their private information exposed.
From this graph, $\epsilon = 0.085$ can be chosen as another reasonable privacy value because $\delta^*$ is only a miniscule $0.007$.
At that $\epsilon$, there are only 6 individuals for which $\delta_i$ is nonzero, and only 2 for which it is  greater than $0.001$.
The estimated probability density with and without the greatest $\delta_i$ individual is shown in Figure~\ref{fig.worst_density}.
Observing such graphs provides clarity as to how in particular an individual contributes to the distribution of query results.

Thus, the strength of empirical differential privacy is that it gives data controllers an intuitive way to understand the risk of privacy breaches given a particular query.
In this case, it seems that releasing the average annual precipitation will maintain at least $(0.12, 0)$-empirical differential privacy,
so a data controller can feel reassured that providing answers to the query will minimally expose individuals' information.
Note that this example featured a theoretically unbounded average as the query, and databases in which all individuals are always present.
However, neither of these need be the case.
Empirical differential privacy works on arbitrary stochastic queries and allows each component database to have missing individuals.
Lastly, the only trade-off in this example is the value for $\eps$ vs $\delta$.
This is unlike classical differential privacy, which features a privacy vs utility tradeoff.
Empirical differential privacy does allow for this, however.
One could add noise to the query results in a manner analogous to the classical approach to achieve greater privacy.
We explain the process of augmenting the stochastic query with additional noise in Section~\ref{sec.noise}.

\section{Acknowledgements}

Thank you to Andrew Trask, Ishaan Nerurkar, Alex Rozenshteyn and Oana Niculaescu for helpful discussions.

\bibliography{empirical_updated}
\bibliographystyle{apalike}

\newpage
\section{Appendix}

This appendix details a proof for Proposition~\ref{prop.bandwidth}.
Before we provide the proof, we require a few lemmas.
The first is an elementary inequality.
\begin{lem}
	\label{lem.logsumexp}
	Given a set of real numbers $\{x_1, \ldots, x_n\}$, it is true that
	\[
		\log \left( \sum_{i=1}^n \exp(x_i) \right) \leq \max_{\forall i} x_i +\log(n)
	\]
\end{lem}

In proving Proposition~\ref{prop.bandwidth}, an intermediate step implies a need to evaluate a function at all real values for two separate sets of points.
The following lemma replaces that equation with an equivalent one that requires simply comparing the two sets of points.
\begin{lem}
	Let $A$ and $B$ denote finite sets of real numbers.
	Let $d(x, A)$ denote the distance of a point $x$ to the nearest member of set $A$ (resp. $B$).
	Recall the Hausdorff distance between $A$ and $B$,

	\begin{gather*}
            d_H (A,B)  = \max \left( \max_{\forall a \in A} \min_{b \in B} |a - b|, \max_{\forall b \in B} \min_{a \in A} |a - b|,\right).
	\end{gather*}

	We have the equality,

	\[
		\max_{x \in \mathbb{R}} \left\vert d(x,A)-d(x,B) \right\vert = d_H(A,B)
	\]
\end{lem}
This Lemma is proven, for example, in~\cite{conci2018distance}.

	We are now ready to prove Proposition~\ref{prop.bandwidth}, which we've restated first for convenience.
	\bandwidth*
	\begin{proof}
    \begin{align*}
        \max_{\forall x \in \mathbb{R}} \vert \log \alpha(x) - \log \beta(x) \vert &= \\
    &= \max_{\forall x} \left\vert \log \left( \frac{\frac{1}{2\lambda n}\sum_{i=1}^{n} \exp \left( -\frac{|x-a_i|}{\lambda} \right)}
        {\frac{1}{2\lambda n}\sum_{i=1}^{n} \exp \left( -\frac{|x-b_i|}{\lambda} \right)} \right) \right\vert \\
    &= \max_{\forall x} \left\vert \log \left( \frac{1}{n} \right) + \log \left( \frac{\sum_{i=1}^{n} \exp \left( -\frac{|x-a_i|}{\lambda} \right)}
        {\frac{1}{n}\sum_{i=1}^{n} \exp \left( -\frac{|x-b_i|}{\lambda} \right)} \right) \right\vert \\
    &\leq \max_{\forall x} \left\vert \log \left( \frac{1}{n} \right) + \log \left( \frac{\sum_{i=1}^{n} \exp \left( -\frac{|x-a_i|}{\lambda} \right)}
        {\exp \left( -\frac{|x-b_x|}{\lambda} \right)} \right) \right\vert \text{ ($b_x$ is nearest to~$x$) } \\
    &= \max_{\forall x} \left\vert \log \left( \sum_{i=1}^{n} \exp \left( \frac{1}{\lambda}
        \left( |x-b_x|-|x-a_i| \right) \right) \right) - \log(n) \right| \\
    &\leq \max_{\forall x} \left\vert \max_{\forall i} \left( \frac{1}{\lambda}
        \left( |x-b_x|-|x-a_i|  \right)  \right) \right\vert \text{(by LogSumExp)}\\
    &= \frac{1}{\lambda} \max_{\forall x} \left\vert d(x, B) - d(x, A) \right\vert \\
	&= \frac{1}{\lambda} d_H (A,B) \\
    &= \epsilon \\
    \end{align*}
    \end{proof}
\end{document}